\documentclass[a4paper,UKenglish,cleveref, autoref, thm-restate]{lipics-v2021}
\usepackage{times}

\usepackage{verbatim}
\usepackage{graphicx,color}
\def\eps{\epsilon}
\def\suchthat{\;:\;}
\def\given{\;|\;}

\def\err{\mathrm{err}}
\def\ind{\mathrm{ind}}

\newcommand{\abs}[1]{\left|#1\right|}

\newcommand{\norm}[1]{\left\|#1\right\|}

\newcommand{\size}[1]{\left|#1\right|}
\newcommand{\linspan}[1]{\operatorname{span}\left(#1\right)}

\newcommand{\expect}[2]{\underset{#1}{\operatorname{E}}\left[#2\right]}
\newcommand{\expectilde}[2]{\underset{#1}{\operatorname{\tilde{E}}}\left[#2\right]}

\usepackage[linesnumbered,ruled,vlined]{algorithm2e}

\usepackage[normalem]{ulem}
\usepackage{bbm}
\usepackage{float}
\usepackage{soul}
\usepackage{amssymb,amsmath,amsthm}
\allowdisplaybreaks

\usepackage{tikz,lmodern}
\usepackage[most]{tcolorbox}
\usepackage{paralist}

\newcommand{\nnz}{\mathrm{nnz}}
\newcommand{\poly}{\mathrm{poly}}

\newcommand{\R}{\mathbb{R}}

\newcommand\numberthis{\addtocounter{equation}{1}\tag{\theequation}}

\newcommand{\Z}{\mathbf{Z}}

\SetCommentSty{mycommfont}

\SetKwInput{KwInput}{Input}                
\SetKwInput{KwOutput}{Output}

\title{One-pass additive-error subset selection for $\ell_{p}$ subspace approximation}

 \titlerunning{One-pass additive-error subset selection for $\ell_{p}$ subspace approximation}  

\author{Amit Deshpande}{Microsoft Research, Bengaluru, India}{amitdesh@microsoft.com}{}{} \author{Rameshwar Pratap}{Indian Institute of Technology, Mandi, H.P., India}{rameshwar.pratap@gmail.com}{}{}

\authorrunning{Amit Deshpande \and Rameshwar Pratap}

\Copyright{CC-BY} 



\category{} 

\relatedversion{}  
\acknowledgements{We would like to thank anonymous reviewers for their careful reading of our manuscript and their insightful comments and suggestions.}

\nolinenumbers 
\EventEditors{Miko{\l}aj Boja\'{n}czyk, Emanuela Merelli, and David P. Woodruff}
\EventNoEds{3}
\EventLongTitle{49th International Colloquium on Automata, Languages, and Programming (ICALP 2022)}
\EventShortTitle{ICALP 2022}
\EventAcronym{ICALP}
\EventYear{2022}
\EventDate{July 4--8, 2022}
\EventLocation{Paris, France}
\EventLogo{}
\SeriesVolume{229}
\ArticleNo{64}

\begin{document}

\maketitle

\begin{abstract}
We consider the problem of subset selection for $\ell_{p}$ subspace approximation, 
that is, to efficiently find a \emph{small} subset of data points such that solving the problem optimally for this subset gives a good approximation to solving the problem optimally for the original input. Previously known subset selection algorithms based on volume sampling and adaptive sampling~\cite{DeshpandeV07}, for the general case of $p \in [1, \infty)$, require multiple passes over the data. In this paper, we give a one-pass subset selection with an additive approximation guarantee for $\ell_{p}$ subspace approximation,
for any $p \in [1, \infty)$.  Earlier subset selection algorithms that give a one-pass multiplicative $(1+\eps)$ approximation work under the special cases.  Cohen \textit{et al.}~\cite{CohenMM17} gives a one-pass subset section that offers  multiplicative $(1+\eps)$ approximation guarantee for the special case of $\ell_{2}$ subspace approximation. Mahabadi \textit{et al.}~\cite{MahabadiRWZ20} gives a one-pass \emph{noisy} subset selection with $(1+\eps)$ approximation guarantee for $\ell_{p}$ subspace approximation  when $p \in \{1, 2\}$. Our subset selection algorithm gives a weaker, additive approximation guarantee, but it works for any $p \in [1, \infty)$.


\end{abstract}
 
\section{Introduction}\label{sec:intro}
 In \textit{subset selection} problems, the objective is to pick a small subset of the given data such that solving a problem optimally on this subset gives a good approximation to solving it optimally on the entire data. Many coreset constructions in computational geometry and clustering \cite{feldman2020intro}, sampling-based algorithms for large matrices \cite{FKV}, algorithms for submodular optimization and active learning \cite{wei15} essentially perform subset selection. The main advantage of subset selection lies in its interpretability, for example, in gene expression analysis, we would like to find a representative subset of genes from gene expression data rather than just fitting a subspace to the data~\cite{DrineasMM08,MahoneyD09,WangZ13,mahoney2011randomized,pmlr-v119-ida20a}.
 In several machine learning applications such as document classification, face recognition etc., it is desirable to go beyond dimension reduction alone, and pick a subset of representative items or features~\cite{DBLP:conf/soda/GuruswamiS12,MahoneyD09}. 
 Subset selection has been well studied for many fundamental problems such as $k$-means clustering~\cite{DBLP:conf/soda/ArthurV07,DeshpandeKP20}, low-rank approximation~\cite{FKV,DBLP:conf/approx/DeshpandeV06,DeshpandeRVW06,DBLP:conf/soda/GuruswamiS12} and regression~\cite{DerezinskiW17}, to name a few. In low-rank and subspace approximation, the subset selection approach leads to more interpretable solutions than using SVD or random projections-based results. Therefore, subset selection has been a separate and well-studied problem even within the low-rank approximation and subspace approximation literature~\cite{DBLP:conf/soda/GuruswamiS12,DanWZZR19}. 


In the following, we formally state the $\ell_{p}$ subspace approximation problem for $p \in [1, \infty)$.

\noindent {\bf $\ell_p$ subspace approximation:} In this problem, given a dataset $\mathcal{X}=\{x_{1}, x_{2}, \dotsc, x_{n}\}$ of $n$ points in $\R^{d}$, a positive integer $1 \leq k \leq d$ and a real number $p \in[1,\infty)$, the objective is to find a linear subspace $V$ in $\R^{d}$ of dimension at most $k$ that minimizes the sum of $p$-th powers of the Euclidean distances of all the points to the subspace $V$, that is, to minimize
\begin{align}
    \err_{p}(\mathcal{X}, V):=\sum_{i=1}^{n} d(x_{i}, V)^{p}.\numberthis\label{eq:lp_subspace_approx}
\end{align}
%

Throughout this paper, we use $V^{*}$ to denote the optimal subspace for $\ell_{p}$ subspace approximation.
The optimal solutions are different for different values of $p$ but we do not include that in the notation to keep the presentation simple, as our results hold for any $p \in [1, \infty)$. 

Before stating our results, we first explain what a \emph{small} subset and a \emph{good} approximation means in the context of subset selection for $\ell_{p}$ subspace approximation. 

For $\ell_{p}$ subspace approximation, we consider $n$ and $d$ to be large, $k \ll n, d$, and $p$ to be a small constant. Thus, a \emph{small} subset of $\mathcal{X}$ desired in subset selection has size independent of $n$ and $d$, and is bounded by $\text{poly}(k/\eps)$, where $\eps$ is a parameter that controls the approximation guarantee (as explained later). 
Note that the trivial solution $V=0$ gives $\err_{p}(\mathcal{X}, V) = \sum_{i=1}^{n} \norm{x_{i}}^{p}$. Using the standard terminology from previous work \cite{FKV, DeshpandeRVW06, DeshpandeV07}, an additive approximation guarantee means outputting $V$ such that $\err_{p}(\mathcal{X}, V) \leq \err_{p}(\mathcal{X}, V^{*}) + \eps~ \sum_{i=1}^{n} \norm{x_{i}}^{p}$, whereas a multiplicative approximation guarantee means $\err_{p}(\mathcal{X}, V) \leq (1+\eps)~ \err_{p}(\mathcal{X}, V^{*})$. Most subset selection algorithms for $\ell_{p}$ subspace approximation select a $\text{poly}(k/\eps)$-sized subset of $\mathcal{X}$ such that its span contains a subspace $V$ of dimension at most $k$ that is close enough to $V^{*}$ to obtain the above approximation guarantees.

Our objective in this paper is to propose an efficient, one-pass sampling algorithm that performs subset selection for $\ell_{p}$ subspace approximation for $p\in [1, \infty)$
defined as above.
 We note that the problem of one-pass subset selection for $\ell_p$ subspace approximation has been studied for special values of $p$, for example, Cohen \textit{et. al.}~\cite{CohenMM17} gives one-pass subset selection for $p=2$, Mahabadi \textit{et al.}~\cite{MahabadiRWZ20} suggest one-pass \textit{noisy} subset selection for $p=\{1, 2\}$. To the best of our knowledge this problem has not been studied in generality for $p\in [1, \infty)$. In this work, we consider studying this problem. We state our results as follows.


%


\subsection{Our results} 
Our main technical contribution is a one-pass MCMC-based sampling algorithm that can approximately simulate multiple rounds of adaptive sampling. As a direct application of the above, we get the following results for the $\ell_{p}$ subspace approximation 
problem: For $p \in [1, \infty)$, our algorithm makes only one pass over the given data and outputs a subset of $\poly(k/\eps)^p$ points whose span contains a $k$ dimensional subspace with an additive approximation guarantee for $\ell_{p}$ subspace approximation. This generalizes the well-known squared-length sampling algorithm of Frieze \textit{et al.} \cite{FKV} that gives additive approximation guarantee for $\ell_{2}$ subspace approximation (or low-rank approximation under the Frobenium norm). Even though stronger multiplicative $(1+\eps)$ approximation algorithms for $\ell_{p}$ subspace approximation are known in the previous work, either they cannot do subset selection, or they are not one-pass, or they do not work for all $p \in [1, \infty)$.

\noindent\textbf{Organization of the paper:} 
 In Section~\ref{sec:related_work}, we compare and contrast our result with the state-of-the-art algorithms, and explain the key technical challenges, and workarounds. In Section~\ref{sec:MCMC}, we state our MCMC based subset selection algorithm for subset selection for $\ell_p$ subspace approximation.  In Section~\ref{sec:subspace}, we give theoretical bounds on the sample size and approximation guarantee. Finally, in Section~\ref{sec:conclusion}, we conclude our discussion and state some potential open questions of the paper.

\section{Related work}\label{sec:related_work}
In this section, we discuss related work on sampling and sketching algorithms for $\ell_{p}$ subspace approximation, and do a thorough comparison of our results with the state of the art.

\subsection{Sampling-based $\ell_{p}$ subspace approximation} 
Frieze \textit{et al.} \cite{FKV} show that selecting a subset of $O(k/\eps)$ data points as an \textit{i.i.d.} sample from $x_{1}, x_{2}, \dotsc, x_{n}$ picked by squared-length sampling, i.e., $x_{i}$ is picked with probability proportional to $\norm{x_{i}}_{2}^{2}$, gives an additive approximation for $\ell_{2}$ subspace approximation (also known as low-rank approximation under the Frobenius norm). Squared-length sampling can be implemented in one pass over $\mathcal{X}$ using reservoir sampling \cite{vitter1985reservoir, EfraimidisS16weighted}. It is known how to improve the additive approximation guarantee to a multiplicative approximation by combining two generalizations of squared-length sampling, namely, adaptive sampling and volume sampling \cite{DeshpandeRVW06, DeshpandeV07} but it requires $O(k \log k)$ passes over the data. In adaptive sampling, we pick points with probability proportional to the distance from the span of previously picked points, and in volume sampling, we pick a subset of points with probability proportional to the squared volume of the parallelepiped formed by them. Volume sampling a subset of size $k$ can itself be simulated with an approximation factor $k!$ in $k$ rounds of adaptive sampling \cite{DeshpandeRVW06}. For $p=2$, it is also known that picking a subset of $O(k/\eps)$ points by volume sampling gives a bi-criteria $(1+\eps)$ approximation for $\ell_{2}$ subspace approximation \cite{DBLP:conf/soda/GuruswamiS12}. For general $p \in [1, \infty)$, it is known that subset selection based on adaptive sampling and volume sampling can be generalized to get a $(1+\eps)$ multiplicative approximation for $\ell_{p}$ subspace approximation, for any $p \in [1, \infty)$, where the subset is of size $O\left((k/\eps)^{p}\right)$ and it is picked in $O(k \log k)$ passes over the data \cite{DeshpandeV07}.  The main bottleneck for implementing this in one pass is the inability to simulate multiple rounds of adaptive sampling in a single pass.

The only known workarounds to get one-pass subset selection for $\ell_{p}$ subspace approximation are known for the special cases $p=1$ and $p=2$. Cohen \textit{et al.} \cite{CohenMM17} give a one-pass subset selection algorithm with a multiplicative $(1+\eps)$ approximation guarantee for $\ell_{2}$ subspace approximation based on ridge leverage score sampling. Their one-pass implementation crucially uses deterministic matrix sketching \cite{GhashamiLPW15} to approximate the SVD and ridge leverage scores, and works only for $p=2$, to the best of our knowledge. Braverman \textit{et al.} \cite{BravermanDMMUWZ20} give online algorithms for $\ell_{2}$ subspace approximation (or low-rank approximation) via subset selection but their subset size $O(\frac{k}{\epsilon} \log n \log \kappa)$ is not independent on $n$ and depends logarithmically on the number of points $n$ and the condition number $\kappa$. Recent work by Mahabadi \textit{et al.} \cite{MahabadiRWZ20} gives a one-pass algorithm with a multiplicative $(1+\eps)$ approximation guarantee for $\ell_{p}$ subspace approximation. However, their algorithm works only in the special cases $p \in \{1, 2\}$ and it outputs a subset of noisy data points instead of the actual data points.

A different objective for $\ell_{p}$ subspace approximation has also been studied in literature \cite{BanBBKLW19,ChierichettiG0L17}, namely, minimizing the entry-wise $\ell_p$-norm low-rank approximation error. To state it formally, given an input matrix $A \in \R^{n\times d}$ and a real number $p \in [0, \infty)$, their objective is to find a matrix $B$ of rank at most $k$ that minimizes $\sum_{i, j}|A_{i,j}-B_{i,j}|^p$.

\subsection{Sketching-based $\ell_{p}$ subspace approximation} Sketching-based algorithms compute a sketch of a given data in a single pass, using which one can compute an approximately optimal solution to a given problem on the original data. The problem of $\ell_{p}$ subspace approximation has been well-studied in previous work on sketching algorithms. However, a limitation of these results is that they do not directly perform subset selection. We mention a few notable results as follows: For $p=2$, extending deterministic matrix sketching of Liberty \cite{Liberty13}, Ghashami \textit{et al.} \cite{GhashamiP14,GhashamiLPW16} give a deterministic one-pass sketching algorithm that gives a multiplicative $(1+\eps)$ approximation guarantee for $\ell_{2}$ subspace approximation (or low-rank approximation under the Frobenius norm). Cormode \textit{et al.} \cite{CormodeDW18} extend the above deterministic sketching idea to $p \neq 2$ and give a $\poly(k)$ approximation for entry-wise $\ell_{1}$-norm low-rank approximation and an additive $\eps~ \norm{b}_{\infty}$ approximation for $\ell_{\infty}$ regression. There is another line of work based on sketching algorithms using random projection. Random projection gives a multiplicative $(1+\eps)$ approximation for $\ell_{2}$ subspace approximation in running time $O(\nnz(X) \cdot \text{poly}(k/\eps))$, subsequently improved to a running time of $O(\text{nnz}(X) + (n+d)\cdot \text{poly}(k/\eps))$ by Clarkson and Woodruff \cite{ClarksonW13}. Feldman \textit{et al.} \cite{FeldmanMSW10} also give a one-pass algorithm for multiplicative $(1+\eps)$ approximation for $\ell_{p}$ subspace approximation, for $p\in [1,2]$. However, these results do not provide a one-pass subset selection.


\subsection{Comparison with other MCMC-based sampling results} Theorem $4$ of Anari \textit{et al.}~\cite{anari2016monte} gives a MCMC based sampling algorithm to approximate volume sampling distribution. However, their algorithm requires a greedy algorithm to pick the initial subset that requires $k$ passes over the input.

The MCMC sampling has also been explored in the context of $k$-means clustering. The $D^2$-sampling proposed by Arthur and Vassilvitskii~\cite{DBLP:conf/soda/ArthurV07} adaptively samples $k$ points -- one point in each passes over the input, and the sampled points give $O(\log k)$ approximation to the optimal clustering solution. The results due to~\cite{DBLP:conf/nips/BachemLH016, BachemLHK16} suggest generating MCMC sampling distribution by taking only one pass over the input that closely approximates the underlying $D^2$ sampling distribution, and offer close to the optimal clustering solution. Building on these MCMC based sampling techniques,  Pratap \textit{et al.}~\cite{DBLP:conf/acml/PratapDND18} gives one pass subset section for \textit{spherical} $k$-means clustering~\cite{DBLP:journals/ml/DhillonM01}.

 \section{MCMC sampling algorithm}\label{sec:MCMC}
In this section, we state our MCMC based sampling algorithm for subset selection for $\ell_p$ subspace approximation.  We first recall the adaptive sampling algorithm\cite{DeshpandeRVW06,DeshpandeV07} for $\ell_p$ subspace approximation.

Adaptive sampling~\cite{DeshpandeRVW06,DeshpandeV07} \textit{w.r.t.} a subset $S \subseteq \mathcal{X}$ is defined   as picking points from $\mathcal{X}$ such that the probability of picking any point $x \in \mathcal{X}$ is proportional to $d(x, \linspan{S})^{p}$. We denote this probability by

\begin{equation}
  p_{S}(x) = \frac{d(x, \linspan{S})^{p}}{\err_{p}(\mathcal{X}, S)}, \quad \text{for $x \in \mathcal{X}$}.  
\end{equation}
For any subset $S$ whose $\err_{p}(\mathcal{X}, S)$ is not too small, we show that adaptive sampling \textit{w.r.t.} $S$ can be approximately simulated by an MCMC sampling algorithm that only has access to \textit{i.i.d.} samples of points $x \in \mathcal{X}$ picked from the following easier distribution:
\begin{equation}\label{eq:mixed_distribution_span}
    q(x) =  \frac{d(x, \linspan{\tilde{S}})^{p}}{2~\err_{p}(\mathcal{X}, \tilde{S})} + \frac{1}{2\size{\mathcal{X}}},
\end{equation}
for some initial subset $\tilde{S}$. We give the above definition of $q(x)$ using an arbitrary initial or \emph{pivot} subset $\tilde{S}$ because it will be useful in our analysis of multiple rounds of adaptive sampling. However, our final algorithm uses a fixed subset $\tilde{S} = \emptyset$ such that
\begin{equation}\label{eq:mixed_distribution}
    q(x) = \frac{\norm{x}_{2}^{p}}{2\sum_{x \in \mathcal{X}} \norm{x}_{2}^{p}} + \frac{1}{2\size{\mathcal{X}}}.
\end{equation}
 Note that sampling from this easier distribution, namely, picking $x \in \mathcal{X}$ with probability $q(x)$ (mentioned in Equation~\eqref{eq:mixed_distribution}), can be done in only one pass over $\mathcal{X}$ using weighted reservoir sampling \cite{Chao1982}. 
Weighted reservoir sampling keeps a reservoir of finite items, and for every new item, calculates its relative weight to randomly decide if the item should be added to the reservoir. If the new item is selected, then one of the existing items from the reservoir is picked uniformly and replaced with the new item. Further, given any non-negative weights $w_x$, for each point $x \in \mathcal{X}$, weighted reservoir sampling can pick an \textit{i.i.d.} sample of points, where $x$ is picked with probability proportional to its weight $w_x$. Note that this does not require the knowledge of $\sum_{x \in \mathcal{X}} w_x$. Thus, we can run two reservoir sampling algorithms in parallel to maintain two samples, one that picks points with probability proportional to $||x||_2^p$, and another that picks points with uniform probability. Our actual sampling with probability proportional $q(x)=\tfrac{\norm{x}_{2}^{p}}{2\sum_{x \in \mathcal{X}} \norm{x}_{2}^{p}} + \tfrac{1}{2\size{\mathcal{X}}}$ picks from one of the two reservoirs with probability $1/2$ each. 
Therefore, our MCMC algorithm uses a single pass of $\mathcal{X}$ to pick a small sample of \textit{i.i.d.} random points from the probability distribution $q(\cdot)$, in advance. Note that $q(\cdot)$ is an easier and fixed distribution compared to $p_{S}(\cdot)$. The latter one depends on $S$ and could change over multiple rounds of adaptive sampling.

Let $x \in \mathcal{X} $ be a random point sampled with probability $q(x)$. Consider a random walk whose single step is defined as follows: sample another point $y \in \mathcal{X}$ independently with probability $q(y)$ and sample a real number $r$ \textit{u.a.r.} from the interval $(0, 1)$, and if 
\[  
\frac{d(y, \linspan{S})^{p}~ q(x)}{d(x, \linspan{S})^{p}~ q(y)} = \frac{p_{S}(y)~ q(x)}{p_{S}(x)~ q(y)} > r,
\]
then move from $x$ to $y$, else, stay at $x$. Essentially, this does rejection sampling using a simpler distribution $q(\cdot)$. Observe that the stationary distribution of the above random walk is the adaptive sampling distribution $p_{S}(\cdot)$. We use $\tilde{P}_{m}^{(1)}(\cdot \given S)$ to denote the resulting distribution on $\mathcal{X}$ after $m$ steps of the above random walk. Note that $m$ steps of the above random walk can be simulated by sampling $m$ \textit{i.i.d.} points from the distribution $q(\cdot)$ in advance, and representing them implicitly as $m$-dimensional points.

\begin{tcolorbox}[float=t!,colback=black!5!white,colframe=black!75!black]
\textbf{One-pass (approximate MCMC) adaptive sampling algorithm:}

\textbf{Input:} a discrete subset $\mathcal{X} \subseteq \R^{d}$ and integer parameters $t, l, m \in \Z_{\geq 0}$. \\
\textbf{Ouput:} a subset $S \subseteq \mathcal{X}$. 
\begin{enumerate}
\item Pick an \textit{i.i.d.} sample $\mathcal{Y}$ of size $\size{\mathcal{Y}} = ltm$ from $\mathcal{X}$, without replacement, where the probability of picking $x \in \mathcal{X}$ is
\[
q(x) =  \frac{d(x, \linspan{\tilde{S}})^{p}}{2~\err_{p}(\mathcal{X}, \tilde{S})} + \frac{1}{2\size{\mathcal{X}}}.
\]
We use the \emph{pivot} subset $\tilde{S} = \emptyset$ so the corresponding distribution is 
\[
q(x) = \dfrac{1}{2}\dfrac{\norm{x}_{2}^{p}}{\sum_{x \in \mathcal{X}} \norm{x}_{2}^{p}} + \dfrac{1}{2 \size{\mathcal{X}}}.
\]
\texttt{\%\% This can be implemented in one pass over $\mathcal{X}$ using weighted reservoir sampling \cite{Chao1982}. Weighted reservoir sampling is a weighted version of the classical reservoir sampling where the probability of inclusion of an item in the sample is proportional to the weight associated with the item.}
\item Initialize $S \leftarrow \emptyset$.
\item  For $i=1, 2, \dotsc, l$ do:
\begin{enumerate}
\item Pick an \textit{i.i.d.} sample $A_{i}$ of size $\size{A_{i}} = t$ from $\mathcal{X}$ as follows. Each point in $A_{i}$ is sampled by taking $m$ steps of the following random walk starting from a point $x$ picked with probability $q(x)$. In each step of the random walk, we pick another point $y$ from $\mathcal{X}$ with probability $q(y)$ and pick a real number $r$ uniformly at random from the interval $(0,1)$. If $\dfrac{d(y, \linspan{S})^{p}~ q(x)}{d(x, \linspan{S})^{p}~ q(y)} > r$ then move to $y$, else, stay at the current point. \\
\texttt{\%\%
Note that  we add only the final point obtained after the $m$-step random walk in the subset $S$.}\\
\texttt{\%\%
We note that  the steps $1$-$3$ of the algorithm can be simulated by taking only one pass over the input as discussed below. Suppose we have a single-pass Algorithm $A$ for sampling from a particular distribution, we can design another Algorithm $B$ that runs in parallel to Algorithm $A$ and post-processes its sample. In our setting, once we know how to get an \textit{i.i.d.} sample of points, where point $x$ is picked with probability $q(x)$, we can run another parallel thread that simulates a random walk whose each step requires a point picked with probability $q(x)$ and performs Step $3$.}
\item $S \leftarrow S \cup A_{i}$.\\
\end{enumerate}
\item Output $S$.
\end{enumerate}
\end{tcolorbox}

\begin{tcolorbox}[float=t!,colback=black!5!white,colframe=black!75!black]
\textbf{One-pass MCMC $\ell_{p}$ subspace approximation algorithm:}

\textbf{Input:} a discrete subset $\mathcal{X} \subseteq \R^{d}$, an integer parameter $k \in \Z_{\geq 0}$ and an error parameter $\delta \in \R_{\geq 0}$. \\
\textbf{Output:} a subset $\mathcal{S} \subseteq \mathcal{X}$ of $\tilde{O}\left((k/\epsilon)^{p+1} \right)$ points. 
\begin{enumerate}
    \item Repeat the following $O(k \log \frac{1}{\eps})$ times in parallel and pick the best sample,  $\mathcal{S}$  that minimizes $\sum_{x \in \mathcal{X}} d(x, \linspan{\mathcal{S}})^p.$
    \begin{enumerate}
        \item Call \textbf{One-pass (approximate MCMC) adaptive sampling algorithm} with $t=\tilde{O}((k/\epsilon)^{p+1})$, $l=k$ and $m=1 + \frac{2}{\epsilon_{1}} \log\frac{1}{\epsilon_{2}}$.
    \end{enumerate}
    \item Output $\mathcal{S}.$
\end{enumerate}
 
\end{tcolorbox}\label{box:111}


  Lemma \ref{lem:tv-dist} below shows that for any subsets $\tilde{S} \subseteq S \subseteq \mathcal {X}$ (where $\tilde{S}$ is the initial subset, and $S$ is the current subset), either $\err_{p}(\mathcal{X}, S)$ is small compared to $\err_{p}(\mathcal{X}, \tilde{S})$, or our MCMC sampling distribution closely approximates the adaptive sampling distribution $p_{S}(\cdot)$ in total variation distance. 
 Proof of Lemma~\ref{lem:tv-dist} relies on Corollary $1$ of Cai~\cite{cai} that gives an upper bound on the TV distance between these two distributions in terms of: 1) length of the Markov chain, and 2) upper bound on the ratio between these two distributions for any input point.
 
\begin{lemma} \label{lem:tv-dist}
Let $\eps_1, \eps_2 \in (0, 1)$ and $\tilde{S} \subseteq S \subseteq \mathcal{X}$. Let $P^{(1)}(\cdot \given S)$ denote the distribution over an i.i.d. sample of $t$ points picked from adaptive sampling w.r.t. $S$, and let $\tilde{P}^{(1)}_{m}(\cdot \given \tilde{S})$ denote the distribution over $t$ points picked by $t$ independent random walks of length $m$ each in our one-pass adaptive sampling algorithm; see step 3(a). Then for $m \geq 1+ \frac{2}{\eps_1}\log \tfrac{1}{\eps_2}$, either $\err_{p}(\mathcal{X}, S) \leq \eps_{1}~ \err_{p}(\mathcal{X}, \tilde{S})$ or $\norm{P^{(1)}(\cdot \given S) - \tilde{P}^{(1)}_{m}(\cdot \given S)}_{TV} \leq \eps_{2}t$.
\end{lemma}
\begin{proof}
First, consider the $l=1, t=1$ case of the one-pass adaptive sampling algorithm described above. In this case, the procedure outputs only one element of $\mathcal{X}$. This random element is picked by $m$ steps of the following random walk starting from an $x$ picked with probability $q(x)$. In each step, we pick another point $y$ with probability $q(y)$ and sample a real number $r$ \textit{u.a.r.} from the interval $(0, 1)$, and if $p_{S}(y) q(x)/p_{S}(x) q(y) > r$, then we move from $x$ to $y$, else, we stay at $x$. Observe that the stationary distribution of the above random walk is the adaptive sampling distribution \textit{w.r.t.} $S$ given by $p_{S}(x) = d(x, \linspan{S})^{p}/\err_{p}(\mathcal{X}, S)$. Using Corollary $1$ of \cite{cai}, 
 the total variation distance after $m$ steps of the random walk is bounded by  
\[
\left(1 - \frac{1}{\gamma}\right)^{m-1} \leq e^{-(m-1)/\gamma} \leq \eps_{2},~ \text{where $\gamma = \max_{x \in \mathcal{X}} \frac{p_{S}(x)}{q(x)}$}.
\]
The above bound is at most $\eps_{2}$ if we choose to run the random walk for $m \geq 1+\gamma \log \frac{1}{\eps_2}$ steps. Now suppose $\err_{p}(\mathcal{X}, S) > \eps_{1}~ \err_{p}(\mathcal{X}, \tilde{S})$. Then, for any $x \in \mathcal{X}$
\begin{align*}
\frac{p_{S}(x)}{q(x)} & = \dfrac{\dfrac{d(x, \linspan{S})^{p}}{\err_{p}(\mathcal{X}, S)}}{\dfrac{1}{2}\dfrac{d(x, \linspan{\tilde{S}})^{p}}{\err_{p}(\mathcal{X}, \tilde{S})} + \dfrac{1}{2 \size{\mathcal{X}}}} \\
& \leq \dfrac{2~ d(x, \linspan{S})^{p}~ \err_{p}(\mathcal{X}, \tilde{S})}{d(x, \linspan{\tilde{S}})^{p}~ \err_{p}(\mathcal{X}, S)}  \leq \frac{2}{\eps_{1}},
\end{align*}
using $d(x, \linspan{S})^{p} \leq d(x, \linspan{\tilde{S}})^{p}$ because $\tilde{S} \subseteq S$, and the above assumption $\err_{p}(\mathcal{X}, S) > \eps_{1}~ \err_{p}(\mathcal{X}, \tilde{S})$. Therefore, $m > \frac{2}{\eps_{1}} \log\frac{1}{\eps_{2}}$ ensures that $m$ steps of the random walk gives a distribution within total variation distance $\eps_{2}$ from the adaptive sampling distribution for picking a single point.

Note that for $t > 1$ both the adaptive sampling and the MCMC sampling procedure pick an \textit{i.i.d.} sample of $t$ points, so the total variation distance is additive in $t$, which means
\[
\norm{P^{(1)}(\cdot \given S) - \tilde{P}^{(1)}_{m}(\cdot \given S)}_{TV} \leq \eps_{2}t,
\]
assuming $\err_{p}(\mathcal{X}, S) > \eps_{1}~ \err_{p}(\mathcal{X}, \tilde{S})$. This completes a proof of the lemma.
\end{proof}
\section{$\ell_{p}$ subspace approximation}\label{sec:subspace}
In this section, we give our result for one pass subset selection for $\ell_p$ subspace approximation. We first show (in Lemma~\ref{lemma:induction}) that the true adaptive sampling can be well approximated by one pass (approximate) MCMC based sampling algorithm. Building on this result, in Proposition~\ref{prop:adaptive-iter-p} and Theorem~\ref{thm:final_MCMC}, we show bounds on the number of steps taken by the Markov chain, and on the sample size that gives an additive approximation for the $\ell_p$ subspace approximation. Our MCMC-based sampling ensures that our problem statement's single-pass subset selection criteria are satisfied.

First, let's set up the notation required to analyze the true adaptive sampling as well as our one-pass (approximate MCMC) adaptive sampling algorithm. For any fixed subset $S \subseteq \mathcal{X}$, we define
\begin{align}
\err_{p}(\mathcal{X}, S) & = \sum_{x \in \mathcal{X}} d(x, \linspan{S})^{p}, \\
P^{(1)}(T|S) & = \prod_{x \in T} \frac{d(x, \linspan{S})^{p}}{\err_{p}(\mathcal{X}, S)}, \\
& \qquad \text{for any subset $T$ of size $t$}, \nonumber\\
\expect{T}{\err_{p}(\mathcal{X}, S \cup T)} & = \sum_{T \suchthat \size{T} = t} P^{(1)}(T \given S)~ \err_{p}(\mathcal{X}, S \cup T). \numberthis\label{eq:exp_T}
\end{align}

Given a subset $S \subseteq \mathcal{X}$, $P^{(1)}(T \given S)$ denotes the probability of picking a subset $T \subseteq \mathcal{X}$ of \textit{i.i.d.} $t$ points by adaptive sampling \textit{w.r.t.} $S$. We use $P^{(l)}(T_{1:l}|S)$ to denote the probability of picking a subset $T_{1:l} = B_{1} \cup B_{2} \cup \dotsc \cup B_{l} \subseteq \mathcal{X}$ of $tl$ points by $l$ iterative rounds of adaptive sampling, where in the first round we sample a subset $B_{1}$ consisting of \textit{i.i.d.} $t$ points \textit{w.r.t.} $S$, in the second round we sample a subset $B_{2}$ consisting of \textit{i.i.d.} $t$ points \textit{w.r.t.} $S \cup B_{1}$, and so on to pick $T_{1:l} = B_{1} \cup B_{2} \cup \dotsc \cup B_{l}$ over $l$ iterations. Similarly, in the context of adaptive sampling, we use $T_{2:l}$ to denote $B_{2} \cup \dotsc \cup B_{l}$. We abuse the notation $\expect{T_{1:l} \given S}{\cdot}$ to denote the expectation over $T_{1:l}$ picked in $l$ iterative rounds of adaptive sampling starting from $S$.

Given a \emph{pivot} subset $\tilde{S} \subseteq \mathcal{X}$ and another subset $S \subseteq \mathcal{X}$ such that $\tilde{S} \subseteq S$, consider the following MCMC sampling with parameters $l, t, m$ that picks $l$ subsets $A_{1}, A_{2}, \dotsc, A_{l}$ of $t$ points each, where $m$ denotes the number of steps of a random walk used to pick these points. This sampling can be implemented in a single pass over $\mathcal{X}$, for any $l, t, m$, and any given subsets $\tilde{S} \subseteq S$. For $T_{1:l} = A_{1} \cup A_{2} \cup \dotsc \cup A_{l}$.  We use $\tilde{P}^{(l)}_{m}(T_{1:l} \given S)$ to denote the probability of picking $T_{1:l}$ as the output of the following MCMC sampling procedure. Similarly, in the context of MCMC sampling, we use $T_{2:l}$ to denote $A_{2} \cup \dotsc \cup A_{l}$. We abuse the notation $\expectilde{T_{1:l} \given S}{\cdot}$ to denote the expectation over $T_{1:l}$ picked using the MCMC sampling procedure starting from $S$ with a pivot subset $\tilde{S} \subseteq S$.

We require the following additional notation in our analysis of the above MCMC sampling. We use $\tilde{P}^{(1)}_{m}(T \given S)$ to denote the resulting distribution over subsets $T$ of size $t$, when we use the above sampling procedure with $l=1$. We define the following expressions:
\begin{align*}
\ind_{p}(\mathcal{X}, S) & = \mathbbm{1}\left(\err_{p}(\mathcal{X}, S) \leq \epsilon_{1}~ \err_{p}(\mathcal{X}, \tilde{S})\right), \numberthis \label{eq:indicator_p}\\    
\expectilde{T}{\err_{p}(\mathcal{X}, S \cup T)} & = \sum_{T \suchthat \size{T} = t} \tilde{P}^{(1)}_{m}(T \given S)~ \err_{p}(\mathcal{X}, S \cup T), \numberthis \label{eq:tilde_expectation} \\    
\expectilde{T}{\ind_{p}(\mathcal{X}, S \cup T)} & = \sum_{T \suchthat \size{T} = t} \tilde{P}^{(1)}_{m}(T \given S)~ \ind_{p}(\mathcal{X}, S \cup T).  \numberthis\label{eq:tilde_indicator}  
\end{align*}
  The expression $\ind_{p}(\mathcal{X}, S)$ (in Equation~\eqref{eq:indicator_p}) denotes an indicator random variable that takes value $1$ if error \textit{w.r.t.} subset $S$ is smaller than $\epsilon_1$ times error \textit{w.r.t.} subset $\tilde {S}$, and $0$ otherwise.  The expression $\expectilde{T}{\err_{p}(\mathcal{X}, S \cup T)}$ (in Equation~\eqref{eq:tilde_expectation})
denotes the expected error over the subset $T$ picked using the MCMC sampling procedure starting from the set $S$ such that the initial subset $\tilde{S} \subseteq S$.

%
%
%

Now Lemma \ref{lemma:induction} analyzes the effect of starting with an initial subset $S_{0}$ and using the same $S_{0}$ as a pivot subset for doing the MCMC sampling for $l$ subsequent iterations of adaptive sampling, where we pick $t$ \textit{i.i.d.} points in each iteration using $t$ independent random walks of $m$ steps. Lemma \ref{lemma:induction} shows that the expected error for subspace approximation after doing the $l$ iterations of adaptive sampling is not too far from the expected error for subspace approximation after replacing the $l$ iterations with MCMC sampling.
\begin{lemma} \label{lemma:induction}
For any subset $S_{0} \subseteq \mathcal{X}$, any $\eps_{1}, \eps_{2} \in (0, 1)$ and any positive integers $t, l, m$ with $m \geq 1+\frac{2}{\eps_1}\log \frac{1}{\eps_2}$,
\begin{align*}
& \expectilde{T_{1:l} \given S_{0}}{\err_{p}(\mathcal{X}, S_{0} \cup T_{1:l})} \leq \expect{T_{1:l} \given S_{0}}{\err_{p}(\mathcal{X}, S_{0} \cup T_{1:l})} + \left(\epsilon_{1} + \epsilon_{2} t l\right) \err_{p}(\mathcal{X}, S_{0}).
\end{align*}
\end{lemma}
\begin{proof}
We show a slightly stronger inequality than the one given above, i.e., for any $S_{0}$ such that $\tilde{S} \subseteq S_{0}$,
\begin{align*}
 \expectilde{T_{1:l} \given S_{0}}{\err_{p}(\mathcal{X}, S_{0} \cup T_{1:l})} 
& \leq \expect{T_{1:l} \given S_{0}}{\err_{p}(\mathcal{X}, S_{0} \cup T_{1:l})} \\
& \quad + \left(\epsilon_{1} \expectilde{T_{1:l} \given S_{0}}{\ind_{p}(\mathcal{X}, S_{0} \cup T_{1:l})} + \epsilon_{2} t l\right) \err_{p}(\mathcal{X}, \tilde{S}).
\end{align*}
The special case $S_{0} = \tilde{S}$ gives the lemma. We prove the above-mentioned stronger statement by induction on $l$. For $l=0$, the above inequality holds trivially. Now assuming induction hypothesis, the above holds true for $l-1$ iterations (instead of $l$) starting with any subset $S_{1} = S_{0} \cup A \subseteq \mathcal{X}$ because $\tilde{S} \subseteq S_{0} \subseteq S_{1}$. 
\begin{align}
& \expectilde{T_{1:l} \given S_{0}}{\err_{p}(\mathcal{X}, S_{0} \cup T_{1:l})} \nonumber \\
& = \expectilde{S_{1} \given S_{0}}{\expectilde{T_{2:l} \given S_{1}}{\err_{p}(\mathcal{X}, S_{1} \cup T_{2:l})}} \nonumber \\
& = \sum_{S_{1} \suchthat \ind_{p}(\mathcal{X}, S_{1}) = 1} \tilde{P}^{(1)}_{m}(S_{1} \given S_{0})~ \expectilde{T_{2:l} \given S_{1}}{\err_{p}(\mathcal{X}, S_{1} \cup T_{2:l})} \nonumber \\
& \qquad+ \sum_{S_{1} \suchthat \ind_{p}(\mathcal{X}, S_{1}) = 0} \tilde{P}^{(1)}_{m}(S_{1} \given S_{0})~ \expectilde{T_{2:l} \given S_{1}}{\err_{p}(\mathcal{X}, S_{1} \cup T_{2:l})}. \label{eq:parts-by-ind}
\end{align}
If $\ind_{p}(\mathcal{X}, S_{1}) = 1$ then $\err_{p}(\mathcal{X}, S_{1} \cup T_{2:l}) \leq \err_{p}(\mathcal{X}, S_{1}) \leq \eps_{1}~ \err_{p}(\mathcal{X}, S_{0})$, so the first part of the above sum can be bounded as follows.
\begin{align}
& \sum_{S_{1} \suchthat \ind_{p}(\mathcal{X}, S_{1}) = 1} \tilde{P}^{(1)}_{m}(S_{1} \given S_{0})~ \expectilde{T_{2:l} \given S_{1}}{\err_{p}(\mathcal{X}, S_{1} \cup T_{2:l})} \nonumber \\
& \leq \eps_{1}~ \err_{p}(\mathcal{X}, S_{0}) \cdot   \sum_{S_{1} \suchthat \ind_{p}(\mathcal{X}, S_{1}) = 1} \tilde{P}^{(1)}_{m}(S_{1} \given S_{0})~ \expectilde{T_{2:l} \given S_{1}}{\ind_{p}(\mathcal{X}, S_{1} \cup T_{2:l})}. \label{eq:part-1-ind}
\end{align}
We give an upper bound on the second part as follows.
\begin{align}
& \sum_{S_{1} \suchthat \ind_{p}(\mathcal{X}, S_{1}) = 0} \tilde{P}^{(1)}_{m}(S_{1} \given S_{0})~ \expectilde{T_{2:l} \given S_{1}}{\err_{p}(\mathcal{X}, S_{1} \cup T_{2:l})} \nonumber \\
& = \sum_{S_{1} \suchthat \ind_{p}(\mathcal{X}, S_{1}) = 0} \tilde{P}^{(1)}_{m}(S_{1} \given S_{0})~ \expectilde{T_{2:l} \given S_{1}}{\err_{p}(\mathcal{X}, S_{1} \cup T_{2:l})}. \nonumber \\
& \leq \sum_{S_{1} \suchthat \ind_{p}(\mathcal{X}, S_{1}) = 0} \tilde{P}^{(1)}_{m}(S_{1} \given S_{0})~ \cdot \nonumber \\
& \quad \left(\expect{T_{2:l} \given S_{1}}{\err_{p}(\mathcal{X}, S_{1} \cup T_{2:l})}+ (\eps_{1}~ \expectilde{T_{2:l} \given S_{1}}{\ind_{p}(\mathcal{X}, S_{1} \cup T_{2:l})} + \eps_{2} t (l-1))~ \err_{p}(\mathcal{X}, \tilde{S})\right). \numberthis \label{eq:intermediateEq} \\
 & \qquad \left(\text{by applying the induction hypothesis to $(l-1)$ iterations starting from $S_{1}$}. \right)\nonumber \\
& \leq \sum_{S_{1} \suchthat \ind_{p}(\mathcal{X}, S_{1}) = 0} P^{(1)}(S_{1} \given S_{0})~ \expect{T_{2:l} \given S_{1}}{\err_{p}(\mathcal{X}, S_{1} \cup T_{2:l})} \nonumber \\
& \qquad + \eps_{1}~ \err_{p}(\mathcal{X}, \tilde{S})~ \cdot  \sum_{S_{1} \suchthat \ind_{p}(\mathcal{X}, S_{1}) = 0} \tilde{P}^{(1)}_{m}(S_{1} \given S_{0})~ \expectilde{T_{2:l} \given S_{1}}{\ind_{p}(\mathcal{X}, S_{1} \cup T_{2:l})} \nonumber \\
& \qquad + \eps_{2} t (l-1)~ \err_{p}(\mathcal{X}, \tilde{S})~ \sum_{S_{1} \suchthat \ind_{p}(\mathcal{X}, S_{1}) = 0} \tilde{P}^{(1)}_{m}(S_{1} \given S_{0}) \nonumber \\
& \qquad + \sum_{S_{1} \suchthat \ind_{p}(\mathcal{X}, S_{1}) = 0} \abs{\tilde{P}^{(1)}_{m}(S_{1} \given S_{0}) - P^{(1)}(S_{1} \given S_{0})}~ \cdot  \expect{T_{2:l} \given S_{1}}{\err_{p}(\mathcal{X}, S_{1} \cup T_{2:l})}. \nonumber \\
& \qquad \left(\text{by adding and subtracting the term $\sum_{S_{1} \suchthat \ind_{p}(\mathcal{X}, S_{1}) = 0} P^{(1)}(S_{1} \given S_{0})~ \expect{T_{2:l} \given S_{1}}{\err_{p}(\mathcal{X}, S_{1} \cup T_{2:l})}$ in Eq.~\eqref{eq:intermediateEq}.
}\right) \nonumber \\
& \leq \sum_{S_{1}} P^{(1)}(S_{1} \given S_{0})~ \expect{T_{2:l} \given S_{1}}{\err_{p}(\mathcal{X}, S_{1} \cup T_{2:l})} \nonumber \\
& \qquad + \eps_{1}~ \err_{p}(\mathcal{X}, \tilde{S})~ \sum_{S_{1} \suchthat \ind_{p}(\mathcal{X}, S_{1}) = 0} \tilde{P}^{(1)}_{m}(S_{1} \given S_{0})~ \cdot  \expectilde{T_{2:l} \given S_{1}}{\ind_{p}(\mathcal{X}, S_{1} \cup T_{2:l})} \nonumber \\
& \qquad + \eps_{2} t (l-1)~ \err_{p}(\mathcal{X}, \tilde{S}) + \sum_{S_{1} \suchthat \ind_{p}(\mathcal{X}, S_{1}) = 0} \abs{\tilde{P}^{(1)}_{m}(S_{1} \given S_{0}) - P^{(1)}(S_{1} \given S_{0})}~ \cdot \err_{p}(\mathcal{X}, \tilde{S}). \nonumber \\
& \qquad \left(\text{by upper bounding the probability expression   $\sum_{S_{1} \suchthat \ind_{p}(\mathcal{X}, S_{1}) = 0} \tilde{P}^{(1)}_{m}(S_{1} \given S_{0})$ by $1$.
}\right) \nonumber \\
& \leq  \expect{T_{1:l} \given S_{0}}{\err_{p}(\mathcal{X}, S_{0} \cup T_{1:l})} \nonumber \\
& \qquad + \eps_{1}~ \err_{p}(\mathcal{X}, \tilde{S})~ \sum_{S_{1} \suchthat \ind_{p}(\mathcal{X}, S_{1}) = 0} \tilde{P}^{(1)}_{m}(S_{1} \given S_{0})~ \cdot  \expectilde{T_{2:l} \given S_{1}}{\ind_{p}(\mathcal{X}, S_{1} \cup T_{2:l})} \nonumber \\
& \qquad + \eps_{2} t (l-1)~ \err_{p}(\mathcal{X}, \tilde{S})  + \norm{\tilde{P}^{(1)}(\cdot \given S_{0}) - P^{(1)}(\cdot \given S_{0})}_{TV}~ \err_{p}(\mathcal{X}, \tilde{S}). \nonumber \\
& \qquad \left(\text{as $\expect{T_{1:l} \given S_{0}}{\err_{p}(\mathcal{X}, S_{0} \cup T_{1:l})}=\sum_{S_{1}} P^{(1)}(S_{1} \given S_{0})~ \expect{T_{2:l} \given S_{1}}{\err_{p}(\mathcal{X}, S_{1} \cup T_{2:l})}$ by Eq.~\eqref{eq:exp_T}.
}\right) \nonumber \\
& \leq \expect{T_{1:l} \given S_{0}}{\err_{p}(\mathcal{X}, S_{0} \cup T_{1:l})} \nonumber \\
& \qquad + \eps_{1}~ \err_{p}(\mathcal{X}, \tilde{S})~ \sum_{S_{1} \suchthat \ind_{p}(\mathcal{X}, S_{1}) = 0} \tilde{P}^{(1)}_{m}(S_{1} \given S_{0})~ \cdot  \expectilde{T_{2:l} \given S_{1}}{\ind_{p}(\mathcal{X}, S_{1} \cup T_{2:l})} \nonumber \\
& \qquad + \eps_{2} t (l-1)~ \err_{p}(\mathcal{X}, \tilde{S})  + \eps_{2} t~ \err_{p}(\mathcal{X}, \tilde{S}). \label{eq:part-2-ind} 
\end{align}
Finally, Equation~\eqref{eq:part-2-ind} holds using Lemma \ref{lem:tv-dist} about the total variation distance between $P^{(1)}$ and $\tilde{P}^{(1)}$ distributions. Plugging the bounds $\eqref{eq:part-1-ind}$ and $\eqref{eq:part-2-ind}$ into $\eqref{eq:parts-by-ind}$, we get
\begin{align*}
& \expectilde{T_{1:l} \given S_{0}}{\err_{p}(\mathcal{X}, S_{0} \cup T_{1:l})} \\
& \leq \expect{T_{1:l} \given S_{0}}{\err_{p}(\mathcal{X}, S_{0} \cup T_{1:l})} + \eps_{1}~ \err_{p}(\mathcal{X}, \tilde{S})~ \sum_{S_{1}} \tilde{P}^{(1)}_{m}(S_{1} \given S_{0})~ \cdot \expectilde{T_{2:l} \given S_{1}}{\ind_{p}(\mathcal{X}, S_{1} \cup T_{2:l})} \nonumber \\
& \quad+ \eps_{2} t (l-1)~ \err_{p}(\mathcal{X}, \tilde{S})  + \eps_{2} t~ \err_{p}(\mathcal{X}, \tilde{S}). \\
& = \expect{T_{1:l} \given S_{0}}{\err_{p}(\mathcal{X}, S_{0} \cup T_{1:l})}  +  \left(\epsilon_{1} \expectilde{T_{1:l} \given S_{0}}{\ind_{p}(\mathcal{X}, S_{0} \cup T_{1:l})} + \epsilon_{2} tl\right)~ \err_{p}(\mathcal{X}, \tilde{S}). \\
& \leq \expect{T_{1:l} \given S_{0}}{\err_{p}(\mathcal{X}, S_{0} \cup T_{1:l})} +  \left(\epsilon_{1} + \epsilon_{2} t l\right)~ \err_{p}(\mathcal{X}, \tilde{S}),
\end{align*}
which completes the proof of Lemma \ref{lemma:induction}.
\end{proof}

Theorem 5 from \cite{DeshpandeV07} shows that in $l=k$ rounds of adaptive sampling, where in each round we pick $t=\tilde{O}\left((k/\eps)^{p+1}\right)$ points and take their union, gives an additive approximation guarantee for $\ell_{p}$ subspace approximation with probability at least $1/2k$. Repeating it multiple times and taking the best can boost the probability further. We restate the main part of this theorem below.

\begin{proposition} \label{prop:adaptive-iter-p} (Theorem 5, \cite{DeshpandeV07}) Let $k$ be any positive integer, let $\eps \in (0, 1)$ and $S_{0} = \emptyset$. Let $l=k$ and $t=\tilde{O}\left((k/\eps)^{p+1}\right)$. If $S_{l} = S_{0} \cup T_{1:l}$ is obtained by starting from $S_{0}$ and doing  adaptive sampling according to the $p$-th power of distances in $l$ iterations, and in each iteration we add $t$ points from $\mathcal{X}$, then we have $\size{S_{l}} = tl = \tilde{O}(k \cdot (k/\epsilon)^{p+1})$ such that
\begin{align*}
& \err_{p}(\mathcal{X}, S_{0} \cup T_{1:l})^{1/p} \leq \err_{p}(\mathcal{X}, V^{*})^{1/p} + \eps~ \err_{p}(\mathcal{X}, \emptyset)^{1/p},
\end{align*}
with probability at least $1/2k$, and where $V^{*}$ minimizes $\err_{p}(\mathcal{X}, V)$ over all linear subspaces $V$ of dimension at most $k$. If we repeat this $O(k \log \frac{1}{\eps})$ times then the probability of success can be boosted to $1 - \eps$.
\end{proposition}
Combining Lemma \ref{lemma:induction} and Proposition \ref{prop:adaptive-iter-p} we get the following Theorem.

\begin{theorem}\label{thm:final_MCMC}
For any positive integer $k$, any $p \in [1, \infty)$, and any $\delta \in \R_{\geq 0}$, starting from $S_{0} = \emptyset$ and setting the following parameters in one-pass MCMC $\ell_{p}$ subspace approximation algorithm (see Section \ref{sec:MCMC})
\begin{align*}
\epsilon & = \delta/4, \\
\eps_{1} & = \delta^{p}/2^{p+1}, \\
\eps_{2} & = \delta^{p}/2^{p+1} tl, \\ 
m & = 1 + \frac{2}{\delta^{p}} \log\frac{k}{\delta^{p}}, \\
t & = \tilde{O}((k/\eps)^{p+1}), \\
l & = k,
\end{align*}
we get a subset $\mathcal{S}$ of size $\tilde{O}(k \cdot (k/\delta)^{p+1})$ with an additive approximation guarantee  on its expected error as $\err_{p}(\mathcal{X}, V^{*})^{1/p} + \delta~ \err_{p}(\mathcal{X}, \emptyset)^{1/p}$.   Further, the running time of the algorithm is $nd+k\cdot \tilde{O}\left(\left(\frac{k}{\delta} \right)^{p+1} \right).$
\end{theorem}
\begin{proof}
From Lemma \ref{lemma:induction} we know that
\[
\expectilde{T_{1:l} \given \emptyset}{\err_{p}(\mathcal{X}, T_{1:l})} \leq \expect{T_{1:l} \given \emptyset}{\err_{p}(\mathcal{X}, T_{1:l})} + \left(\epsilon_{1} + \epsilon_{2} t l\right) \err_{p}(\mathcal{X}, \emptyset).
\]
Thus, for $p \in [1, \infty)$ we have
\begin{align*}
\expectilde{T_{1:l} \given \emptyset}{\err_{p}(\mathcal{X}, T_{1:l})}^{1/p} & \leq \expect{T_{1:l} \given \emptyset}{\err_{p}(\mathcal{X}, T_{1:l})}^{1/p} + \left(\epsilon_{1} + \epsilon_{2} t l\right)^{1/p} \err_{p}(\mathcal{X}, \emptyset)^{1/p}. \\
& \leq (1-\epsilon) \left(\err_{p}(\mathcal{X}, V^{*})^{1/p} + \eps~ \err_{p}(\mathcal{X}, \emptyset)^{1/p}\right) + \eps~ \err_{p}(\mathcal{X}, \emptyset)^{1/p} \\
& \qquad \qquad + \left(\epsilon_{1} + \epsilon_{2} t l\right)^{1/p} \err_{p}(\mathcal{X}, \emptyset)^{1/p}. \\
& \qquad \qquad \left(\text{using Proposition \ref{prop:adaptive-iter-p}}.\right) \\
& \leq \err_{p}(\mathcal{X}, V^{*})^{1/p} + \left(2\eps +  \left(\epsilon_{1} + \epsilon_{2} t l\right)^{1/p}\right) \err_{p}(\mathcal{X}, \emptyset)^{1/p}. \\
& \leq \err_{p}(\mathcal{X}, V^{*})^{1/p} + \delta~ \err_{p}(\mathcal{X}, \emptyset)^{1/p},
\end{align*}
using $\eps = \delta/4$, $\eps_{1} = \delta^{p}/2^{p+1}$ and $\eps_{2} = \delta^{p}/2^{p+1} tl$.

  We now give a bound on the running time of our algorithm. We require $nd$ time to generate the probability distribution $q(x)$, for $x\in \mathcal{X}$. Further, the running time of MCMC sampling step is~ $t\cdot m \cdot l=k\cdot \tilde{O}\left(\left(\frac{k}{\delta} \right)^{p+1} \right)$. Therefore, the overall running time of the algorithm is $nd+k \cdot\tilde{O}\left(\left(\frac{k}{\delta} \right)^{p+1} \right)$.
\end{proof}

\section{Conclusion and open questions}\label{sec:conclusion}

In this work, we give an efficient one-pass MCMC algorithm that does subset selection with additive approximation guarantee for $\ell_{p}$ subspace approximation, for any $p \in [1, \infty)$. Previously this was only known for the special case of $p=2$~\cite{CohenMM17}. For general case $p \in [1, \infty)$, adaptive sampling algorithm due to~\cite{DeshpandeV07} requires taking multiple passes over the input.  Coming up with a one-pass subset selection algorithm that offers stronger multiplicative guarantees for $p \in [1, \infty)$ remains an interesting open problem.
 
\bibliographystyle{plainurl}
\bibliography{reference}
\end{document}